\documentclass[conference]{IEEEtran}

\IEEEoverridecommandlockouts
\usepackage{cite}
\usepackage{amsmath,amssymb,amsfonts}
\usepackage{algorithmic}
\usepackage{graphicx}
\usepackage{textcomp}
\usepackage{xcolor}

\def\BibTeX{{\rm B\kern-.05em{\sc i\kern-.025em b}\kern-.08em
    T\kern-.1667em\lower.7ex\hbox{E}\kern-.125emX}}
   
\usepackage[nolist,printonlyused]{acronym}
\usepackage[utf8]{inputenc}
\usepackage{comment}
\usepackage{amsmath}
\usepackage{amssymb}
\usepackage{amsthm}
\usepackage{bm}
\usepackage{geometry,tabularx}
\usepackage{tikz}
\usepackage{graphicx}
\graphicspath{ {./images/} }
\usepackage{amsmath,amssymb}
\usepackage{fullpage}
\usepackage{color}   
\usepackage{hyperref}
\hypersetup{
    colorlinks=true, 
    linktoc=all,     
    linkcolor=blue,  
}

\newtheorem{theorem}{Theorem}

\newtheorem{proposition}{Proposition}

\newtheorem{lemma}{Lemma}

\newcommand{\maxstarempty}{\mathrm{max}^*}

\begin{document}

\title{Approximation Capabilities of Neural Networks using Morphological Perceptrons and Generalizations\\
\thanks{This work is supported in part by the National Science Foundation (CCF-1763747).}
}

\author{\IEEEauthorblockN{William Chang, Hassan Hamad, Keith M.~Chugg}
\IEEEauthorblockA{\textit{Ming Hsieh Department of Electrical and Computer Engineering} \\
\textit{University of Southern California}\\
Los Angeles, United States of America \\
\{chan087, hhamad, chugg\}@usc.edu}
}

\maketitle

\begin{abstract}
Standard artificial neural networks (ANNs) use sum-product or multiply-accumulate node operations with a memoryless nonlinear activation.  These neural networks are known to have universal function approximation capabilities.  Previously proposed morphological perceptrons use max-sum, in place of sum-product, node processing and have promising properties for circuit implementations.  In this paper we show that these max-sum ANNs do not have universal approximation capabilities.  Furthermore, we consider proposed signed-max-sum and max-star-sum generalizations of morphological ANNs and show that these variants also do not have universal approximation capabilities.  We contrast these variations to log-number system (LNS) implementations which also avoid multiplications, but do exhibit universal approximation capabilities. 
\end{abstract}


\begin{acronym}
\acro{ANN}{Artificial Neural Network}
\acro{ANNs}{Artificial Neural Networks}
\acro{LNS}{Logarithmic Number System}
\acro{SNN}{Spiking Neural Network}
\acro{MAC}{Multiply-Accumulate}
\end{acronym}

\section{Introduction}

Deep neural networks are driving the AI revolution. They have led to breakthroughs in various fields ranging from computer vision \cite{krizhevsky-2012-alexnet} and natural language processing \cite{devlin-etal-2019-bert} to protein folding \cite{alphafold-2020} and autonomous driving \cite{nvidia-2016-autonomous}. The current trend is toward larger models. This has motivated recent work on complexity reductions for both inference and training. Complexity reduction techniques such as pruning \cite{han-2016-pruning}, sparsity \cite{kundu-2020-sparsity}, and quantization \cite{raghuraman-2018-quantization} have been proposed. This is particularly important for models deployed on edge devices that have limited memory and computational resources.

\begin{table}[t]
    \centering
    \caption{Different node operations studied in this paper}
    \label{tab:node_def}
    \begin{tabularx}{0.35\textwidth}{c|X}
        Node Operation & Sum-Product Equivalent \\[2ex]
        \hline
         sum-product &  \(\displaystyle z = \sum_{i} (x_i y_i)\)\\[3ex]
         
         $\max$-sum & \(\displaystyle z = \bigvee_{i} (x_i + y_i)\)\\[3ex]
         
         signed $\max$-sum & \(\displaystyle z = \bigvee_{i} a_i (x_i + y_i)\)\\[3ex]
         
         $\maxstarempty$-sum & \(\displaystyle z = \underset{i}{\maxstarempty}(x_i + y_i)\)\\[3ex]
         
         LNS & \begin{tabular}[c]{@{}c@{}} \(\displaystyle z = \underset{i}{\maxstarempty_{\pm}}(\log|x_i| + \log|y_i|) \) \\ 
         \(\displaystyle  s_z = s_x \oplus s_y\) \end{tabular}   
    \end{tabularx}
\end{table}

Another approach for complexity reduction is to depart from the \ac{MAC} (or sum-product) processing used in standard \ac{ANNs}. Radically different network structures, such as the Spiking Neural Network (SNN) \cite{roy-2019-snn,cao-2015-snn}, have been proposed. Others have proposed to simply replace sum-product operations with a different, and ideally more efficient, operations \cite{charisopoulos2017morphological,sussner1998signedmaxsum,ritter1996introduction,miyashita-2016-lns,sanyal-2020-lns}. In table~\ref{tab:node_def}, we list the types of node operations that will be the discussed in this paper, where $\vee$ denotes the $\max$ operator. The morphological perceptron replaces the sum-product by a $\max$-sum operation \cite{ritter1996introduction,charisopoulos2017morphological}. The use of the $\max$ function inherently adds non-linearity to the network. The morphological perceptron was extended to use a signed $\max$-sum operation by adding a binary sign parameter in  \cite{sussner1998signedmaxsum}.  In the field of digital communications, and specifically in error correction coding literature, the $\max^*$-sum operation\footnote{$\max^*$-sum is also known as the Jacobian Logarithm \cite{RoViHo95}} is widely used in decoding iterative digital codes \cite{RoViHo95}. This operation can be seen as a natural extension to the $\max$-sum node. Finally,  is well known that the exact equivalent of a sum-product can be implemented in the log domain, i.e. using the \ac{LNS}. \ac{LNS} requires the use of $\max^*_+$-sum and $\max^*_-$-sum processing along with tracking the signs of the linear operands.

A neural network using \ac{LNS} can thus implement the same processing as a standard \ac{ANN}\cite{miyashita-2016-lns,sanyal-2020-lns}. Therefore, \ac{LNS}-based networks inherit the approximation capabilities of standard \ac{ANN}s. Specifically, a single layer network in \ac{LNS} with a non-linear activation is also a universal function approximator\cite{hornik1990universal,hornik1991approximation}. To the best of our knowledge, the approximation capabilities of neural networks with these max-like units have not yet been studied. In this paper, our goal is to characterize the approximation capabilities of neural networks having a $\max$-sum, signed $\max$-sum or a $\max^*$-sum nodes. We prove that these kind of networks are not universal approximators. In addition, we characterize the exact set of functions that they can approximate.

The paper is organized as follows. In section 2, we define our notation and review the universal approximation theorem of neural networks. Section 3 contains the main results of the paper. We conclude in section 4. The proofs of all the theorems are given in the appendices.

\begin{figure*}[tb]
    \centering
    \includegraphics[width=0.8\textwidth]{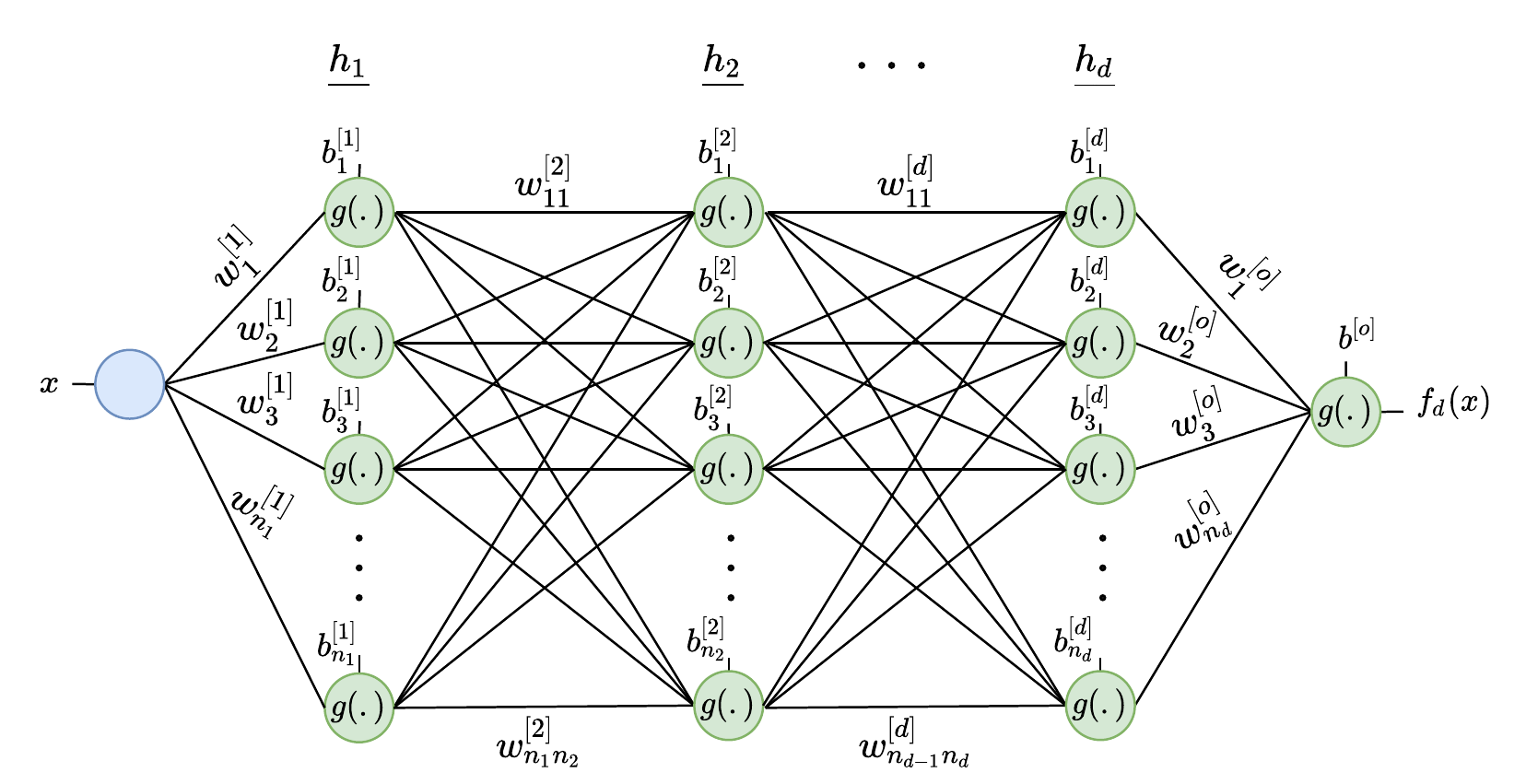}
    \caption{A $d$ hidden layer neural network with scalar input $x$ and scalar output $f_d(x)$. The node processing at each unit is donated by the function $g(.)$. Note that $g(.)$ also includes the activation function, when applicable.}
    \label{fig:network}
\end{figure*}

\section{Background}
Consider a scalar-input and scalar-output fully connected neural network with $d$ hidden layers as shown in Fig.~\ref{fig:network}. Denote the node function at each unit by the $g(.)$. Note that we also lump the activation function, if any, in the definition of $g(.)$. The standard   \ac{MAC}-based node processing $g(.)$ with input  activation vector $\bm{x} = [x_1,x_2,..,x_n]$, weight vector $\bm{w} = [w_1,w_2,..,w_n]$, bias value $b$, and a non-linear activation function $\sigma(.)$, is defined as $g(\bm{x},\bm{w},b) = \sigma(b + \sum_{i=1}^n w_ix_i$).

A well known fact of neural networks is that they are universal approximators. In \cite{hornik1991approximation}, it was shown that a one-hidden layer neural network $f: \mathbb{R}^k \to \mathbb{R}^l$ with a bounded and nonconstant activation function can uniformly approximate any function $h$ on $\mathbb{R}^k$, given that sufficient hidden units are available. This result has been later extended to networks with unbounded activations, such as the rectified linear activation (Relu) \cite{yarotsky2017error}.

A class of neural networks $\mathcal{N}:[0, 1]^n \to \mathbb{R}$ is said to  uniformly  approximate a given function $h: [0, 1]^n \to \mathbb{R}$  iff for any $\epsilon > 0$, there exists a function $N \in \mathcal{N}$ such that $\forall x\in [0, 1]^n: |N(x) - h(x)| < \epsilon$. This class of neural networks is said to be a universal approximator iff it uniformly approximates all continuous functions $h: [0, 1]^n \to \mathbb{R}$.

The universal approximation theorems mentioned above concern neural networks with the standard sum-product or \ac{MAC} processing. In this paper, the approximation capabilities of neural networks with three different non-standard node operations are investigated. 

Our proof method is to show neural networks with the non-\ac{MAC} nodes exhibit bounded first derivatives. By the following lemma, this implies that these networks are not universal approximators. 
\begin{lemma}\label{lemma:helper}
Given a single-input and single-output neural network $f: [c, d] \rightarrow \mathbb{R}$. Suppose that $a \leq f'(x) \leq b$ for $a, b \in \mathbb{R}$. Then $f$ can only uniformly approximate functions $h: [0, 1] \rightarrow \mathbb{R}$ such that $a \leq h'(x) \leq b$ almost everywhere.
\end{lemma}
While this result appears intuitive, the same does not hold for higher order derivatives. That is, if $f^{(n)}(x) \in [a, b]$ for some $n \geq 2$, this does not mean that $f$ can only universally approximate functions $h$ such that $a \leq h^{(n)}(x) \leq b$. A simple counter example is a sum-product network with Relu activation on the hidden layers, which is known to be a universal approximator \cite{yarotsky2017error} but has a bounded second derivative.  In Appendix \ref{prop:counterex}, we show that the second derivative of such a network is zero a.e. 

\section{Approximation Capabilities}
In this section we consider the approximation capabilities of neural networks with max-sum processing in place of standard sum-product processing.  We also consider two generalizations to max-sum processing. We show that these non-\ac{MAC} networks are not universal approximators from $\mathbb{R} \to \mathbb{R}$, and thus they are also not universal approximators for $\mathbb{R}^n \to \mathbb{R}$ for any $n \geq 1$. 

We state the results of each case in this section and provide the proofs in appendices.  
\subsection{Max-Sum Network}
In the case of a \emph{max-sum} network, the node function $g(.)$ in Fig.~\ref{fig:network} is defined as 
\begin{align}
    g(\bm{x},\bm{w},b) &= b \vee (w_1+x_1) \vee \cdots \vee (w_n+x_n) \nonumber \\ 
    &= b \vee \left(\bigvee_{i=1}^n (w_i + x_i)\right)
\end{align}
where $\vee$ is the $\max$ operator, i.e. $ x \vee y = \max(x,y)$.  The max-sum node is also known as the morphological perceptron~\cite{charisopoulos2017morphological}. Note that the \emph{max} operation inherently adds non-linearity to the node and no explicit activation function is used. Next we present our first theorem.

   \begin{theorem}\label{thm:max-sum}
Consider a single-input single-output, $d$ hidden layer neural network with \emph{max-sum} node processing $f_d(x): \mathbb{R}\rightarrow \mathbb{R}$. 

Then $f_d(x) = \max(w, w'+ x)$ for some constants $w, w' \in \mathbb{R}$. 
\end{theorem}

Thus, max-sum node processing results in a very limited class of functions that can be realized. In particular, by lemma~\ref{lemma:helper}, max-sum networks are not universal approximators. In \cite{charisopoulos2017morphological}, max-sum layers were combined with standard sum-product layers to obtain effective classifiers.  

\subsection{Signed Max-Sum Network}

In the case of a \emph{signed max-sum} network, the node function $g(.)$ in Fig.~\ref{fig:network} is defined as 
\begin{align}
    g(\bm{x},\bm{w},b) &= b \vee a_1(w_1 + x_1) \vee \cdots \vee a_n(w_n + x_n) \nonumber \\ 
    &= b \vee \left(\bigvee_{i=1}^n a_i(w_i + x_i)\right)
\end{align}
where $a_i \in \{-1, 1\}$.  Note that the $a_i$'s can either be learnable network parameters or fixed. For our purposes, this choice is irrelevant to the study of the approximation capability of the network since we consider all $a_i \in \{-1, 1\}$. 

 \begin{theorem}\label{thm:signed-max-sum}
Consider a single-input single-output, $d$ hidden layer neural network with \emph{signed max-sum} node processing $f_d(x): \mathbb{R}\rightarrow \mathbb{R}$. 

Then $f'_d(x) \in \{-1, 0, 1\}$ a.e. 
\end{theorem}

By Lemma \ref{lemma:helper} \emph{signed max-sum}  networks are not universal approximators.  In fact, by the above theorem, these networks have very limited approximation capabilties.  
Note also that the \emph{max-sum} node is a special case of the \emph{signed max-sum} where all the $a_i$'s are set to $1$.  Theorem \ref{thm:max-sum} implies that the derivative of max-sum networks is limited to the set $\{0, 1\}$ which is a proper subset of the possible derivatives of signed-max-sum networks.

\subsection{Max*-Sum Network}
To define the \emph{max*-sum} node, first note the definition of the \emph{max*} function:
\begin{equation}
    \maxstarempty(x,y) = \ln\left(e^x + e^y\right)
\end{equation}
which can be nested as follows 
\begin{align}
    \underset{i}{\maxstarempty} x_i &= \maxstarempty(x_1,x_2,..,x_n)  \nonumber \\
    &=\ln\left(e^{x_1} + e^{x_2} + .. + e^{x_n} \right) \nonumber \\
    &= \ln\left(\sum_{i=1}^n e^x_i\right)
\end{align}
In the case of a \emph{max*-sum} network, the node function $g(.)$ in Fig.~\ref{fig:network} is defined as 
\begin{align}
    g(\bm{x},\bm{w},b) &= \sigma(\maxstarempty(b ,\underset{i}{\maxstarempty} x_i+w_i)) \nonumber \\ 
    &= \sigma\left(\ln\left(e^B + \sum_{i= 0}^n e^{x_i+w_i}\right)\right)
\end{align}
where $\sigma(x)$ is any activation function such that $\sigma'(x)\in [0, 1]$, e.g. a Relu.

The \emph{max*-sum} node processing may be viewed as doing arithmetic in the log domain. Consider the standard sum-product $c_i = \sum_{i}a_ib_i$ where $a_i, b_i >0~\forall i$. If we let $C_i = \ln(c_i), B_i = \ln(b_i)$ and $A_i = \ln(a_i)$ then
\begin{align}
    C_i &= \ln\left(\sum_{i}a_ib_i\right) = \ln\left(\sum_{i}e^{A_i+B_i}\right) \nonumber \\
    &= \underset{i}{\maxstarempty}(A_i + B_i)
\end{align}
This indicates that if all data inputs to the network as well as all the weights are non-negative, then working with \emph{max*-sum} nodes is the equivalent of working with the standard sum-product node, but in the log domain.  This is the reason for introducing a nonlinear activation function for \emph{max*-sum} networks.  Specifically, if a linear mapping with non-negative weights, biases, inputs, and outputs is implemented in the log domain, it would be a network with \emph{max*-sum} node processing and no activation function.

 \begin{theorem}\label{thm:maxstar-sum}
Consider a single-input single-output, $d$ hidden layer neural network with \emph{max*-sum} node processing $f_d(x): \mathbb{R}\rightarrow \mathbb{R}$. 

Then $0 \le f'_d(x) < 1$ 
\end{theorem}
Again, by Lemma \ref{lemma:helper} \emph{max*-sum} networks are not universal approximators.  

\subsection{Discussion}
In the previous subsections we proved that a neural network with a \emph{max-sum}, \emph{signed max-sum} or \emph{max*-sum} node processing is not a universal approximator.  We noted that \emph{max*-sum} processing can be viewed as sum-product processing of non-negative values in the log domain.  If this were extended to include signed inputs, weights, and biases, one would need to separately track the effects on the sign and magnitude of the quantities.  In fact, this \ac{LNS} arithmetic.  
A neural network with \ac{LNS} arithmetic, and a non-linear activation function, is equivalent to a log-domain sum-product ANN, and thus is also a universal approximator. This suggests that it is unlikely that there exists an ANN using \emph{max}-like family of computations that is a universal approximator and substantially less complex than an \ac{LNS} network.

\section{Conclusion}
We showed that an ANN with morpholgical perceptron (i.e., max-sum) node processing has a very limited approximation capabilities.  Furthermore, direct extensions to this node processing, such as signed-max-sum and $\max^*$-sum, also are not universal approximators.  This suggests that it is unlikely to achieve universal approximation with  a max-sum generalization substantially simpler than an LNS implementation.

\appendices

 \section{Proof of Lemma \ref{lemma:helper}}
 \label{app:helper}
 
\begin{proof} of Lemma \ref{lemma:helper}

Let $N(x)$ be a neural network that uniformly approximates $h(x):[0,1] \to \mathbb{R}$ to $\epsilon$ accuracy for any $\epsilon > 0$. For the sake of contradiction, suppose that there is a set of nonzero measure $I$ such that $h'(x) > b$ for all $x \in I$. Without any loss of generality, we can assume $I = [c, d]$ is a closed interval. It follows from the condition $h'(x) > b$ that
\begin{equation}
    h(d) = h(c) + b(d-c) + \delta
\end{equation}
for some constant $\delta>0$. Since $N(x)$ is at most $\epsilon$ away from $h(a)$ it follows that $N(c)< h(c)+ \epsilon$. Thus, from the condition $N'(x) \leq b$
\begin{equation}
   N(d) <  h(a)+ \epsilon + (d-c)b.
\end{equation} 

Thus:
\begin{equation}
    h(d) -N(d) > \delta - \epsilon
\end{equation}
When $\epsilon< \frac{\delta}{2}$, $h(d) - N(d)>
\frac{\delta}{2} > \epsilon$, which contradicts the fact that $N(x)$ uniformly approximates $h(x)$ to $\epsilon$ accuracy.

Similarly, it can be shown that $h'(x) \geq a$ a.e.
\end{proof}

\begin{proposition}\label{prop:counterex}
A single-input single-output, sum-product Network with Relu activation has its second derivative equal to $0$ a.e.
\begin{proof}
Let $\sigma = \max(0, x)$ be the ReLu function and $N_d(x)$ be the resulting $d$-hidden layer network. Proceed by induction on the number of hidden layers. For one hidden layer network with width $n$, we have
\begin{equation}
    N_1(x) = \sum_{i=1}^n w^{[o]}_i\sigma(w^{[1]}_i x)
\end{equation}
So that 
\begin{align}
    N_1''(x) &= \sum_{i=1}^n w^{[o]}_i (w^{[1]}_i)^2\sigma''(w^{[1]}_i x)\\
    &= 0
\end{align}
where the last equality holds a.e. since $\sigma''(x) = 0$ except when $x = 0$. For the inductive step, note that if the final hidden layer has width $n$, then
\begin{equation}
    N_d(x) = \sum_{i=1}^n w^{[o]}_i\sigma(g_i(x))
\end{equation}
where $w^{[o]}_i$ are the weights at the output layer, and $g_1(x),...,g_n(x)$ can be thought of sub-networks of $N_d(x)$ with $d - 1$ hidden layers, and thus by the inductive hypothesis satisfy $g''_i(x) = 0$ a.e. The next two derivatives of this network are computed as
\begin{align}
    N'_d(x) &= \sum_{i=1}^n w^{[o]}_i \sigma'(g_i(x))g'_i(x)\\
    N''_d(x) &= \sum_{i=1}^n w^{[o]}_i (\sigma''(g_i(x))g'_i(x)^2 + \sigma'(g_i(x))g''_i(x))
\end{align}
The first term in the summation is $0$ except when $g_i(x) = 0$ which happens finitely many times, and thus this term is equal to $0$ a.e. The second term in the summation is $0$ a.e. by the inductive hypothesis. Thus, $N''_d(x) = 0$ a.e., completing our induction. 
\end{proof}
\end{proposition}

 \section{Proof of Theorem \ref{thm:max-sum}}
 \label{app:morphological}

To prove Theorem \ref{thm:max-sum}, the following lemma is needed,
\begin{lemma}\label{lemma:reduction}
For any constants $a_1, a_2, b_1, b_2$, $\exists w_0, w_1$ such that  $(a_1+ (b_1\vee  x))\vee (a_2+(b_2 \vee x))= w_0\vee (w_1 + x)$

\begin{proof}
We first note that $(a_1+ (b_1\vee x))\vee (a_2+ (b_2 \vee x)) = a_1 + (b_1 \vee x)\vee(a_2 - a_1+ b_2\vee x) = a_1 + (b_1\vee x)\vee (c+ (b_2\vee x))$. Set
\begin{equation}
    g(x) = a_1 + (b_1 \vee x) \vee (c+ (b_2 \vee x))
\end{equation} 

We hope to show that $g(x) = w_0 \vee( w_1 + x)$ for some $w_0, w_1$ which will depend on $a_1, b_1, b_2, c$. We now have the following cases based on these values. 

\emph{Case 1:} $b_1 > b_2$, $c< 0$. When $x \leq b_2$, we have $f(x) = a_1 + b_1 \vee( b_2 + c) = a_1 + b_1$. When $x \in (b_2, b_1)$, we have $f(x) = a_1 + b_1 \vee(x+c) = a_1 + b_1$. Finally, when $x \geq b_1$, we have $f(x) = a_1 + x \vee (x+c) = a_1 + x$. Thus we can set $w_0 = a_1 + b_1$ and $w_1 = a_1$, completing this case. 

\emph{Case 2:} $b_1 > b_2$, $c> 0$. When $x \leq b_2$, we have $f(x) = a_1 + b_1\vee (b_2 + c)$. When $x \in (b_2, b_1)$, we have $f(x) = a_1 + b_1\vee( x+c)$. Finally, when $x \geq b_1$, we have $f(x) = a_1 + x \vee(x+c) = x+c+a_1$. Thus we can set $w_0 =  a_1 + b_1 \vee(b_2 + c)$ and $w_1 = c+a_1$, completing this case. 

Since $ g(x) = a_1 + (b_1 \vee x)\vee(c+ (b_2\vee x)) = a_1 + c + (-c + (b_1\vee x))\vee(b_2\vee x)$,   the lemma also holds when $b_2 \geq b_1$. This completes all the cases.
\end{proof}
\end{lemma}

\begin{proof} of Theorem \ref{thm:max-sum}

We can prove this by induction on the number of hidden layers. Consider a one hidden layer network with $n$ computation units. The resulting function is of the form
\begin{equation}
   N_1(x) = b^{[o]} \vee \left(\bigvee_{i=1}^n\left(w^{[o]}_{i} + b^{[1]}_{i}\vee (w^{[1]}_{i}+x)\right)\right)
\end{equation}
for some real valued constants $b^{[o]}, w^{[o]}_{i}, w^{[1]}_{i}, b^{[1]}_{i}$ for $i \in [1, n]$. First note that $b\vee (w + x) = w + (b - w)\vee x$. Then 
\begin{align}
   N_1(x) = b^{[o]} \vee \Bigg( \Bigg. \bigvee_{i=1}^n\left( \right. &w^{[o]}_{i} + w^{[1]}_{i} +\nonumber \\
 &(b^{[1]}_{i} - w^{[1]}_{i})\vee x \left. \Bigg) \Bigg. \right) \label{eq:N_1}
\end{align} 
We can now apply Lemma \ref{lemma:reduction} to equation \eqref{eq:N_1} repeatedly to complete our base case. 

For the inductive step, suppose that the theorem is true for $d-1$ hidden layers. For a network with $d$ layers, suppose that the $d$-th hidden layer had $n$ computation units. Denote the output of each computation unit in the $d$-th hidden layer as $g_i(x)$ for $i \in [1, n]$. Now $g_i(x)$ can be thought of as a neural network with $d - 1$ hidden layers. It is clear from these definitions that 
\begin{equation}\label{eq:N_d}
      N_d(x) = b^{[o]} \vee \left(\bigvee_{i=1}^n\left(w^{[o]}_{i} + g_i(x)\right)\right)
\end{equation}
for some real valued constants $b^{[o]}, w^{[o]}_{i}$ for $i \in [1, n]$. By the inductive hypothesis, each $g_i(x) = w_i\vee(w_i' + x)$ for some values $w_i$, $w_i'$. As with the base case, Lemma \ref{lemma:reduction} can be repeatedly applied to equation \eqref{eq:N_d} to complete the inductive step. 

\end{proof}

 \section{Proof of Theorem \ref{thm:signed-max-sum}}
 \label{app:modified-mp}

To prove Theorem \ref{thm:signed-max-sum}, we first need the following lemma, analogous to lemma \ref{lemma:reduction}.

\begin{lemma}\label{lemma:reduction_modified}
Consider the class of functions 
\begin{multline}
    \mathcal{F}:= \{c_1, c_2 - x, c_3 + x, c_1 \vee (c_2 - x),  c_1 \vee (c_3 + x),\\
    (c_2 - x) \vee (c_3 + x),c_1 \vee (c_2 - x) \vee (c_3 + x)|c_1,c_2,c_3 \in \mathbb{R} \}
\end{multline}
Then $\forall f, g \in \mathcal{F}$, we have $f\vee g \in \mathcal{F}$. 

\begin{proof}
This is trivial to verify except when $f$ or $g$ takes on the form $  (c_2 - x) \vee (c_3 + x),c_1 \vee (c_2 - x) \vee (c_3 + x$. To address this case, let $\sigma_{a_1,a_2,a_3}(x) =(a_1 - x)\vee  a_2\vee (a_3 + x)$. Then it is sufficient to verify the following properties
\begin{enumerate}
    \item $\sigma_{a_1,a_2,a_3}(x)\vee(b - x) = \sigma_{c_1,c_2,c_3}(x)$ for some $c_1,c_2, c_3 \in \mathbb{R}$. 
    
    \item $\sigma_{a_1,a_2,a_3}(x)\vee( b + x) = \sigma_{c_1,c_2,c_3}(x)$ for some $c_1,c_2, c_3 \in \mathbb{R}$. 
    
    \item  $\sigma_{a_1,a_2,a_3}(x) \vee b = \sigma_{c_1,c_2,c_3}(x)$ for some $c_1,c_2, c_3 \in \mathbb{R}$. 
    
    \item $\sigma_{a_1, a_2, a_3}(x)\vee \sigma_{b_1, b_2, b_3}(x) = \sigma_{c_1,c_2,c_3}(x)$ for some constants $c_1,c_2,$ and $c_3$. 

\end{enumerate}

Property (1) can be proved as follows,
\begin{align}
 &\sigma_{a_1,a_2,a_3}(x)\vee(b - x)\\ \nonumber
    &= ((a_1 - x)\vee( b - x))\vee a_2\vee (a_3 + x)\\ \nonumber
    &= ((a_1\vee b) - x)\vee a_2 \vee(a_3 + x)\\ \nonumber
    &= \sigma_{a_1\vee b , a_2, a_3}(x)
\end{align}
Properties (2) and (3) can be treated similarly. Property (4) follows from basic properties of $a \vee b$,
\begin{align}
    &\sigma_{a_1, a_2, a_3}(x)\vee \sigma_{b_1, b_2, b_3}(x) \\ \nonumber
    &= ((a_1 - x)\vee(b_1 - x))\vee (a_2\vee b_2) \\ \nonumber
    &\qquad \vee ((a_3 +x)\vee(b_3 + x))
\end{align}

Thus, letting $c_1 =a_1 \vee b_1$, $c_2 = a_2 \vee b_2$, $c_3 = a_3 \vee b_3$ gives us the desired result. 
\end{proof}
\end{lemma}

With this lemma we can now prove Theorem \ref{thm:signed-max-sum}.

\begin{proof} of Theorem \ref{thm:signed-max-sum}
Since all the functions in $\mathcal{F}$ as given in Lemma \ref{lemma:reduction_modified} have their derivatives in the set $\{-1, 0, 1\}$ a.e., it is sufficient to use induction on the number of hidden layers to prove that the neural networks given by the theorem statement belong in $\mathcal{F}$. Consider a single input single output, one hidden layer network with $n$ computation units. The resulting function is of the form
\begin{align}\label{eq:N_1_modified}
   N_1(x) = b^{[o]}\vee \Bigg( \Bigg.&\bigvee_{i=1}^n a^{[o]}_{i}\left( \right. w^{[o]}_{i} +  \nonumber \\ 
   &b^{[1]}_{i}\vee a^{[1]}_{i}(w^{[1]}_{i}+ x) \left.\right) \Bigg. \Bigg)
\end{align}

for some real valued constants $b^{[o]}$, $b^{[1]}_{i}$, $a^{[o]}_i$, $a^{[1]}_i$, $w^{[o]}_{i}$, $w^{[1]}_{i}$ for $i \in [1, n]$. Since $a^{[1]}_{i}(w^{[1]}_{i}+ x)$ and $w^{[o]}_{i} \in \mathcal{F}$ as in Lemma \ref{lemma:reduction_modified},  note that $w^{[o]}_{i} + b^{[1]}_{i}\vee a^{[1]}_{i}(w^{[1]}_{i}+ x)\in \mathcal{F}$, from which it is clear that $a^{[o]}_{i}(w^{[o]}_{i} + b^{[1]}_{i}\vee a^{[1]}_{i}(w^{[1]}_{i}+ x))\in \mathcal{F}$ as well. Thus, $N_1(x) \in \mathcal{F}$, completing the base case.

For the inductive step, suppose that the theorem is true for $d-1$ hidden layers. For a network with $d$ layers, suppose the $d$-th hidden layer had $n$ computation units. Denote the output of each computation unit in the $d$-th hidden layer as $g_i(x)$ for $i \in [1, n]$. Now $g_i(x)$ is effectively a neural network with $d - 1$ hidden layers. It is clear from these definitions that 
\begin{equation}\label{eq:N_d_modified}
      N_d(x) = b^{[o]} \vee \left(\bigvee_{i=1}^n a^{[o]}_{i}\left(w^{[o]}_{i} + g_i(x)\right)\right)
\end{equation}
for some real valued constants $b^{[o]}, a^{[o]}_i, w^{[o]}_{i}$ for $i \in [1, n]$. From the inductive hypothesis, $g_i(x) \in \mathcal{F}$, from which a repeated application of Lemma \ref{lemma:reduction_modified} shows that $N_d(x) \in \mathcal{F}$ as well, completing our induction. 
\end{proof}

 \section{Proof of Theorem \ref{thm:maxstar-sum}}
 \label{app:maxstar-sum}
\begin{proof} of Theorem \ref{thm:maxstar-sum}
It is easy to see that $N'_d(x)$ is undefined when $\sigma'(x)$ is undefined. Thus there are only finitely many points for which $N'_d(x)$ is undefined. For the values of $x$ for which $N'_d(x)$ is defined, we shall show that $N'_d(x) \leq 1$ using induction on the number of hidden layers. Consider a single-input single-output, one hidden layer network with $n$ computation units. The resulting function is of the form
\begin{align}
   N_1(x) = \ln \Bigg[ \Bigg. &e^{b^{[o]}} + \nonumber \\
   &\sum_{i=1}^n e^{w^{[o]}_i + \sigma\left(\ln\left(e^{b^{[1]}_i}+e^{w^{[1]}_i+x}\right)\right)} \Bigg. \Bigg] \label{eq:N_1_thm_3}
\end{align} 

for some real valued constants $b^{[o]}, w^{[o]}_i, b^{[1]}_i, w^{[1]}_i$ for $i \in [1, n]$. From equation \eqref{eq:N_1_thm_3} it follows that
\begin{equation}
    N'_1(x) = \frac{\alpha(x)}{\beta(x)}
\end{equation}

where

\begin{align}
&\alpha(x)= \nonumber \\ 
&\sum_{i=1}^n \left[ e^{w^{[o]}_i + \sigma\left(M(x)\right)} 
\sigma'(M(x))\frac{e^{w^{[1]}_i+x}}{e^{b^{[1]}_i}+e^{w^{[1]}_i+x}} \right]
\end{align}

\begin{equation}
    \beta(x) = e^{b^{[o]}} + \sum_{i=1}^n e^{w^{[o]}_i + \sigma\left(M(x)\right)}
\end{equation}

\begin{equation}
  M(x) =\ln\left(e^{b^{[1]}_i}+e^{w^{[1]}_i+x}\right)
\end{equation}

Since $\sigma'(x) \in [0, 1]$, it is immediately clear that $N'_1(x) \geq 0$. $N'_1(x)<1$ follows from the following inequalities. 
\begin{equation}
 e^{w^{[1]}_i+x} < e^{b^{[1]}_i}+e^{w^{[1]}_i+x}
\end{equation}
and
\begin{multline}
    \sum_{i=1}^n \left[e^{w^{[o]}_i + \sigma\left(M(x)\right)}\sigma'(M(x))\right] \\
    < e^{b^{[o]}} + \sum_{i=1}^n e^{w^{[o]}_i + \sigma\left(M(x)\right)}
\end{multline}

completing the base case. For the inductive step, suppose the theorem is true for $d-1$ hidden layers. Consider a network with $d$ hidden layers where the $d$th hidden layer has $n$ computation units. Denote the output of each computation unit in the $d$-th hidden layer as $g_i(x)$ for $i \in [1, n]$. Now $g_i(x)$ is effectively a neural network with $d - 1$ hidden layers. Since there is no activation at the output layer, $N_d(x)$ has the following form
\begin{equation}\label{eq:g_d}
    N_d(x) = \ln\left[e^{b^{[o]}} + \sum_{i=1}^n e^{w^{[o]}_i+\sigma(g_i(x))}\right]
\end{equation}
for some real valued constants $b^{[o]}, w^{[o]}_i$ for $i \in [1, n]$. From equation \eqref{eq:g_d} it follows that
\begin{equation}
    N'_d(x) = \frac{\sum_{i=1}^n e^{w^{[o]}_i+\sigma(g_i(x))}\sigma'(g_i(x))g'_i(x)}{e^{b^{[o]}} + \sum_{i=1}^n e^{w^{[o]}_i+\sigma(g_i(x))}}
\end{equation}

From the inductive hypothesis, $g'_i(x) \geq 0$. Therefore $N'_d(x) \geq 0$. $N'_d(x)<1$ follows from the following inequality which is trivially true since $g'_i(x) < 1$,
\begin{equation}
 e^{w^{[o]}_i+\sigma(g_i(x))}\sigma'(g_i(x))g'_i(x)< e^{w^{[o]}_i+\sigma(g_i(x))}
\end{equation}
This completes the proof.
\end{proof}

\bibliographystyle{IEEEtran}
\bibliography{references} 

\end{document}